\documentclass[journal]{IEEEtran}
\usepackage{amsmath,amsfonts,amssymb}
\usepackage{mathtools}

\usepackage{algorithm}
\usepackage{algorithmic}

\usepackage{array}
\usepackage{booktabs}
\usepackage{multirow}

\usepackage[caption=false,font=normalsize,labelfont=sf,textfont=sf]{subfig}

\usepackage{graphicx}
\usepackage{xcolor}
\usepackage{textcomp}
\usepackage{stfloats}
\usepackage{url}
\usepackage{verbatim}
\usepackage{cite}

\usepackage{amsthm}
\theoremstyle{remark}
\newtheorem{remark}{Remark}
\theoremstyle{plain}
\newtheorem{theorem}{Theorem}
\newtheorem{lemma}{Lemma}
\newtheorem{assumption}{Assumption}

\theoremstyle{remark}
\begin{document}

\title{Active Stiffness Control of a Supportive Continuum Robot}

\author{
Rana Danesh\textsuperscript{a},
Farrokh Janabi-Sharifi\textsuperscript{a,*}, and
Farhad Aghili\textsuperscript{b}%
\thanks{\textsuperscript{a}Rana Danesh and Farrokh Janabi-Sharifi are with the
Department of Mechanical, Industrial, and Mechatronics Engineering,
Toronto Metropolitan University, Toronto, ON, Canada.}%
\thanks{\textsuperscript{b}Farhad Aghili is with the Department of Mechanical,
Industrial and Aerospace Engineering, Concordia University,
Montreal, QC, Canada.}%
\thanks{\textsuperscript{*}Corresponding author: Farrokh Janabi-Sharifi
(email: fsharifi@torontomu.ca).}
}
         



\maketitle

\begin{abstract}

Supportive continuum robots (SCRs) enhance the load-bearing capability of an operative continuum robot by mechanically coupling it with a supportive arm. However, their passive stiffness is determined by the mechanical configuration and cannot be adjusted online for varying payloads or interaction forces. Active stiffness control is therefore needed to regulate the load response and maintain positioning accuracy. Meanwhile, the closed-chain structure introduces kinematic constraints that complicate task-space regulation and stiffness control. This paper presents an active task-space stiffness control framework for a tendon-driven SCR. An existing geometric variable strain model describes the closed-chain dynamics, which are projected onto the constraint-consistent motion subspace. A projected sliding mode controller regulates the operative arm tip while preserving the constraints, and closed-loop stability is established through Lyapunov analysis. After position regulation, active apparent stiffness is introduced through a virtual Cartesian spring based on position-error feedback to shape the force--displacement response. The framework is evaluated in simulation and experimentally validated under prescribed external loads and different desired configurations. Results show that increasing the commanded stiffness gain reduces load-induced tip deflection and increases apparent directional stiffness, thereby improving load resistance and positioning robustness under external loading.

\end{abstract}


\section{Introduction}

\IEEEPARstart{C}{ontinuum} robots (CRs) enable safe physical interaction in confined environments because of their slender structures, high dexterity, and inherent compliance \cite{russo2023continuum,robinson1999continuum,danesh2025backstepping,danesh2026hybrid}. Beyond navigation, many applications of CRs require them to apply forces or support payloads. In such tasks, stiffness is a key factor governing positioning accuracy and payload capacity \cite{russo2021cooperative,janabi2021cosserat}.

Several approaches have been developed to modulate the physical stiffness of CRs through structural, mechanical, and material-based mechanisms. Structural mechanisms include particle jamming \cite{ranzani2015bioinspired,cianchetti2014soft}, layer jamming \cite{kim2013novel,fan2022novel,wang2019electrostatic}, and fiber jamming \cite{brancadoro2020fiber,choi2021tendon}. Other mechanical approaches employ antagonistic tendon, fluidic, or pneumatic actuation \cite{stilli2014shrinkable,shiva2016tendon,giannaccini2018novel}, tendon routing strategies and prescribed tendon displacements \cite{oliver2019continuum}, or tendon locking configurations \cite{kim2017active}. Material-based approaches alter the stiffness of the robot through thermoplastic structures \cite{mcevoy2015thermoplastic}, thermally activated microfluidic materials \cite{balasubramanian2014microfluidic}, low-melting-point alloy mechanisms \cite{shintake2015variable}, and shape memory materials \cite{alambeigi2016continuum}. Although these approaches can substantially change physical stiffness, they may require additional mechanisms, dedicated materials, structural reconfiguration, thermal activation, or increased system complexity \cite{fan2024overview}.

Rather than directly modifying the CR structure or material properties, active stiffness control regulates the apparent stiffness of CRs through sensing, feedback, and actuation. In this case, the controller shapes the closed-loop force--displacement behavior of the robot. Existing studies have investigated stiffness control in task-space \cite{mahvash2011stiffness}, stiffness regulation based on impedance control \cite{duan2024operating}, variable stiffness control using depth vision \cite{lai2021variable}, and simultaneous control of position and stiffness \cite{yi2024simultaneous}. In these approaches, the resulting apparent stiffness depends jointly on the robot mechanics, controller parameters, and interaction conditions. However, most existing studies on both passive stiffness modulation and active stiffness control have primarily focused on single CRs.

Instead of increasing the stiffness of an individual CR, another approach is to distribute external loads through mechanically coupled continuum arms. Supportive architectures achieve this by connecting a supportive arm to an operative arm, as shown in Fig.~\ref{fig:scr_system}(a) \cite{li2025collaborative,lotfavar2017cooperative}. This configuration enables external loads to be shared across the coupled structure, thereby improving the load resistance, payload capacity, effective stiffness, and operational stability of the system while preserving its inherent flexibility \cite{mahoney2016reconfigurable}. Previous studies have shown that the stiffness of supportive continuum robots (SCRs) can be influenced by structural parameters, including the connection point, connection angle, and tendon constraints of the supportive arm \cite{jalali2024dynamic}. However, this approach relies mainly on structural configuration or design changes, rather than active feedback control during operation.

Moreover, an SCR forms a closed-chain system in which the motion and deformation of the two arms are coupled by loop-closure constraints \cite{danesh2026real}. These constraints restrict the allowable motion and generate internal reaction forces at the mechanical connection. Consequently, the control inputs influence the deformation and load distribution of the entire coupled structure. Active stiffness control must therefore shape the task-space force--displacement response while keeping the control action within the constraint-consistent motion subspace.

To the best of our knowledge, active stiffness control has not been investigated for SCRs. This paper addresses this gap by proposing an active task-space stiffness control framework for a closed-chain SCR. The controller first regulates the operative arm end effector to the desired position and then activates an active stiffness term to modify the closed-loop force--displacement response under external loading. The main contributions are as follows:

\begin{itemize}
\item Using the existing geometric variable strain (GVS) dynamic model, the closed-chain SCR dynamics are projected onto the constraint-consistent motion subspace for task-space regulation.

\item A projected sliding mode control law is developed for task-space position regulation while preserving the closed-chain constraints. A Lyapunov analysis establishes closed-loop stability and asymptotic convergence of the task-space errors.

\item An active task-space stiffness control strategy is introduced after position regulation to modify the apparent stiffness of the coupled robot under external loading.

\item The proposed framework is validated through simulation and experiments on a tendon-driven SCR under known external loads and different desired configurations.
\end{itemize}

The remainder of this paper is organized as follows. Section~II presents the GVS model and the closed-chain constraints. Section~III introduces the proposed control framework. Section~IV presents the simulation results, and Section~V provides the experimental validation. Finally, Section~VI concludes the paper.

\section{Geometric Variable Strain Modeling}

Fig.~\ref{fig:scr_system}(a) depicts the SCR schematic configuration considered in this work. The interconnected CRs form a closed kinematic structure. Using the GVS framework, the following establishes the kinematic relations and constrained dynamic model of the system \cite{renda2020geometric,renda2018discrete,armanini2021discrete,danesh2026real}.

\begin{figure}[t]
    \centering
    \captionsetup[subfloat]{font=scriptsize}

    \subfloat[]{
        \includegraphics[
        width=0.22\textwidth,
        trim=9cm 3cm 9cm 0cm,
        clip]{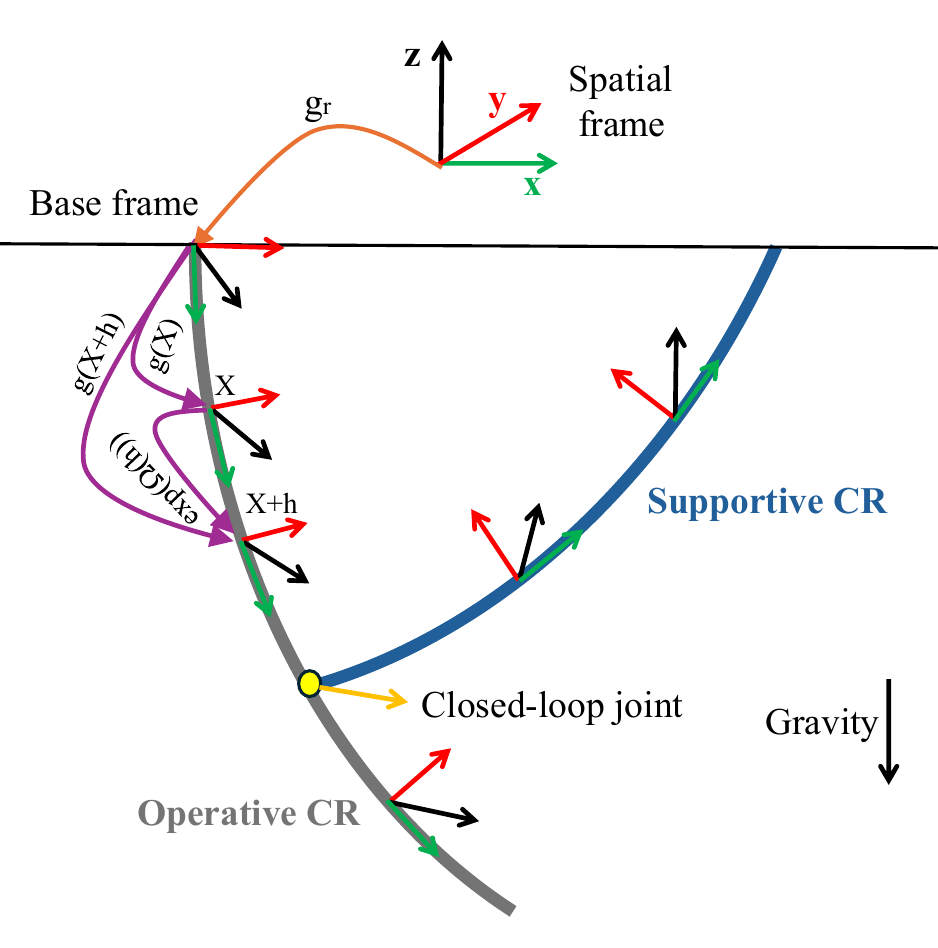}
        \label{fig:scr_schematic}
    }
    \hfill
    \subfloat[]{
        \includegraphics[
        width=0.22\textwidth,
        trim=7.5cm 1cm 7.5cm 0cm,
        clip]{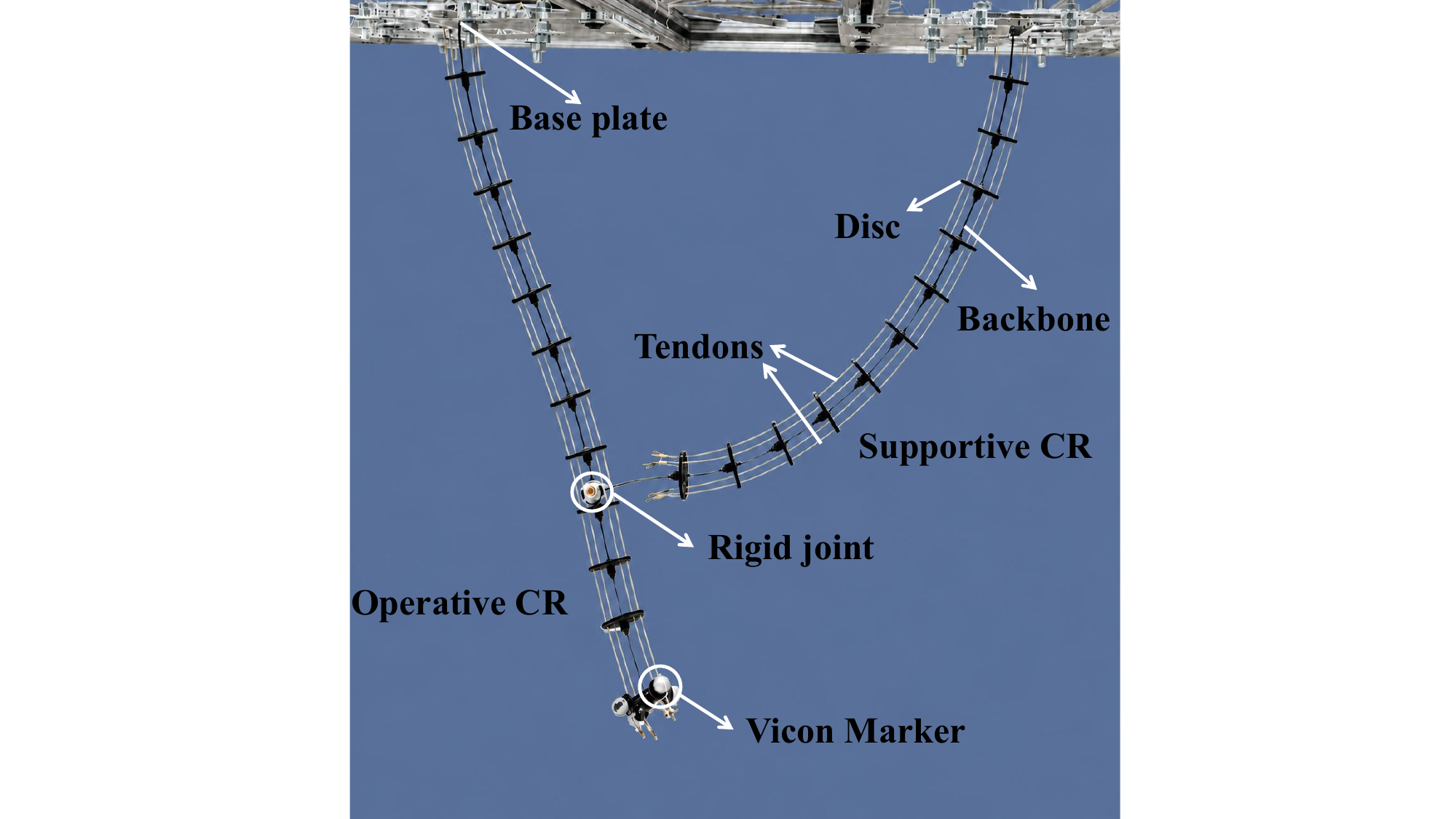}
        \label{fig:scr_setup}
    }

    \caption{Supportive continuum robot system:
(a) schematic representation showing the backbone parameterization and reference frames; and
(b) experimental tendon-driven SCR platform.}
    \label{fig:scr_system}
\end{figure}

\vspace{-1 em}

\subsection{Kinematic Description}
Each CR is treated as a continuously deformable body whose configuration is parameterized by the backbone coordinate {\small $X \in [0,L]$}. The value {\small $X=0$} identifies the base cross section, and {\small $L$} specifies the backbone length. As illustrated in Fig.~\ref{fig:scr_system}(a), a frame is assigned to every backbone cross section. Its pose relative to the spatial frame is represented by the homogeneous transformation
{\small
\begin{equation}
\mathbf{g}(X)=
\begin{bmatrix}
\mathbf{R}(X) & \mathbf{p}(X) \\
\mathbf{0}^{\mathsf{T}} & 1
\end{bmatrix}
\in SE(3),
\label{eq:g_SE3}
\end{equation}}

\noindent where {\small$\mathbf{p}(X)\in\mathbb{R}^3$} specifies the position of the local frame origin, whereas {\small$\mathbf{R}(X)\in SO(3)$} defines the orientation of the corresponding cross section with respect to the spatial frame.

The backbone evolution is formulated on the Lie group {\small$SE(3)$}. Differentiation of the configuration with respect to the material coordinate {\small$X$} and time {\small$t$} gives {\small$\mathbf{g}'=\mathbf{g}\widehat{\boldsymbol{\xi}},
\ \dot{\mathbf{g}}=\mathbf{g}\widehat{\boldsymbol{\eta}}$}, where {\small$(\cdot)'$} and {\small$(\cdot)\dot{;}$} indicate differentiation with respect to {\small$X$} and {\small$t$}, respectively, and {\small$\widehat{(\cdot)}:\mathbb{R}^6\rightarrow\mathfrak{se}(3)$} is the standard hat operator. The strain twist {\small$\boldsymbol{\xi}\in\mathbb{R}^6$} represents local backbone deformation, including bending, torsion, shear, and axial extension per unit length. The body velocity twist {\small$\boldsymbol{\eta}\in\mathbb{R}^6$} describes the instantaneous motion of a cross section in its local frame. Integrating the body kinematic relation {\small$\mathbf{g}'=\mathbf{g}\widehat{\boldsymbol{\xi}}$} over a backbone interval of length {\small$h$}, as indicated in Fig.~\ref{fig:scr_system}(a), produces the Lie-group update {\small$\mathbf{g}(X+h)=\mathbf{g}(X)\exp\big(\widehat{\boldsymbol{\Omega}}(h)\big)$}, where {\small$\boldsymbol{\Omega}(h)\in\mathbb{R}^6$} denotes the Magnus increment associated with the strain field over {\small$[X,X+h]$}. In this study, {\small$\boldsymbol{\Omega}(h)$} is evaluated with a fourth-order Magnus integration method that combines Zanna collocation and two-stage Gauss quadrature~\cite{murray2017mathematical,renda2020geometric}.

Enforcing the equality of mixed partial derivatives,
{\small $(\dot{\mathbf{g}})'=(\mathbf{g}')^{\dot{}}$}, leads to the compatibility relation
{\small
$
\boldsymbol{\eta}'=
\dot{\boldsymbol{\xi}}-
\mathrm{ad}_{\boldsymbol{\xi}}\boldsymbol{\eta},
$
} where {\small $\mathrm{ad}_{(\cdot)}\in\mathbb{R}^{6\times6}$} is the adjoint representation of {\small $\mathfrak{se}(3)$}. Applying the identity {\small $(\mathrm{Ad}_{\mathbf{g}})'=\mathrm{Ad}_{\mathbf{g}}\mathrm{ad}_{\boldsymbol{\xi}}$} gives the following analytical expression for the body velocity

{\small
\begin{equation}
\boldsymbol{\eta}(X)=
\mathrm{Ad}_{\mathbf{g}^{-1}(X)}
\int_{0}^{X}
\mathrm{Ad}_{\mathbf{g}(s)}\dot{\boldsymbol{\xi}}(s)\,ds.
\label{eq:eta_analytic}
\end{equation}
} 

\noindent A finite-dimensional representation is obtained by expanding the strain field with a prescribed set of basis functions,
{\small
\begin{equation}
\boldsymbol{\xi}(X)=
\mathbf{B}_{\boldsymbol{\xi}}(X)\mathbf{q}
+\boldsymbol{\xi}^{\ast}(X),
\label{eq:strain_basis}
\end{equation}
}

\noindent where {\small $\mathbf{B}_{\boldsymbol{\xi}}(X)\in\mathbb{R}^{6\times n}$} is the strain basis matrix, {\small $\mathbf{q}\in\mathbb{R}^n$} is the vector of generalized coordinates, and {\small $n$} is the total number of degrees of freedom. Inserting \eqref{eq:strain_basis} into \eqref{eq:eta_analytic} results in
{\small
\begin{equation}
\begin{aligned}
\boldsymbol{\eta}(X)=&
\mathrm{Ad}_{\mathbf{g}^{-1}(X)}
\int_{0}^{X}
\mathrm{Ad}_{\mathbf{g}(s)}
\mathbf{B}_{\boldsymbol{\xi}}(s)\,ds\;
\dot{\mathbf{q}} 
=
\mathbf{J}(\mathbf{q},X)\dot{\mathbf{q}},
\end{aligned}
\label{eq:geometric_jacobian}
\end{equation}
}

\noindent where {\small $\mathbf{J}(\mathbf{q},X)$} is the geometric Jacobian relating the generalized velocity vector to the local body velocity along the backbone.

\vspace{-1 em}

\subsection{Dynamic Formulation}
The CR dynamics are obtained from Cosserat rod theory expressed on the Lie algebra {\small $\mathfrak{se}(3)$}~\cite{renda2018discrete}. The strong form balance of linear and angular momentum along the backbone is
{\small
\begin{equation}
\boldsymbol{\mathcal{M}}\dot{\boldsymbol{\eta}}
+
\mathrm{ad}^{*}_{\boldsymbol{\eta}}\,\boldsymbol{\mathcal{M}}\boldsymbol{\eta}
=
\boldsymbol{\mathcal{F}}_i'
+
\mathrm{ad}^{*}_{\boldsymbol{\xi}}\,\boldsymbol{\mathcal{F}}_i
+
\bar{\boldsymbol{\mathcal{F}}}_a
+
\bar{\boldsymbol{\mathcal{F}}}_e ,
\label{eq:cosserat_strong}
\end{equation}
}

\noindent where {\small $\boldsymbol{\mathcal{F}}_i$} is the internal wrench, {\small $\bar{\boldsymbol{\mathcal{F}}}_a$} is the distributed wrench induced by actuation, {\small $\bar{\boldsymbol{\mathcal{F}}}_e$} represents the distributed external wrench, and {\small $\boldsymbol{\mathcal{M}}\in\mathbb{R}^{6\times6}$} is the screw inertia matrix. The operator {\small $\mathrm{ad}^{*}_{(\cdot)}$} denotes the coadjoint action on {\small $\mathfrak{se}(3)$}~\cite{renda2020geometric}. The internal wrench {\small $\boldsymbol{\mathcal{F}}_i$} is described by a linear Kelvin--Voigt viscoelastic constitutive model~\cite{renda2014dynamic}, thereby capturing both elastic restoration and strain-rate-dependent dissipation. Under this assumption, the strain and strain-rate fields determine the internal wrench as
{\small
$
\boldsymbol{\mathcal{F}}_i
=
\boldsymbol{\mathcal{K}}\boldsymbol{\xi}
+
\boldsymbol{\mathcal{D}}\,\dot{\boldsymbol{\xi}},
$
} where {\small $\boldsymbol{\mathcal{K}}$} and {\small $\boldsymbol{\mathcal{D}}$} are the stiffness and viscous damping matrices of the rod\cite{hussain2021compliant}. The external loading on the CR includes gravity distributed effects and concentrated loads generated by applied forces or contact. The corresponding external wrench density is
{\small
$
\bar{\boldsymbol{\mathcal{F}}}_e
=
\boldsymbol{\mathcal{M}}
\,\mathrm{Ad}^{-1}_{\mathbf{g}_r\mathbf{g}}\,\boldsymbol{\mathcal{G}}
+
\delta(X-\bar{X})\,\boldsymbol{\mathcal{F}}_p,
$
} where {\small
\(
\boldsymbol{\mathcal{G}}
\)
}is the gravity twist expressed in the spatial frame. The transformation {\small $\mathbf{g}_r\in SE(3)$}, shown in Fig.~\ref{fig:scr_system}(a), maps quantities between the spatial frame and the base frame of the operative arm. The operator {\small $\delta(\cdot)$} is the Dirac distribution, which introduces the concentrated wrench {\small $\boldsymbol{\mathcal{F}}_p$} at the backbone location {\small $X=\bar{X}$}.

\vspace{-0.5 em}

\subsection{Finite Dimensional Dynamics}
To obtain a discrete representation, the strong form in {\small (\ref{eq:cosserat_strong})} is converted into a weak formulation through the principle of virtual work~\cite{renda2018discrete}. Let {\small $\delta\boldsymbol{\zeta}(X)\in\mathbb{R}^6$} be an arbitrary virtual twist field along the backbone. The weak momentum balance is then written as

{\small
\begin{equation}
\begin{aligned}
\int_{0}^{L}
\delta\boldsymbol{\zeta}^{\mathsf{T}}
(
&\boldsymbol{\mathcal{M}}\dot{\boldsymbol{\eta}}
+
\mathrm{ad}^{*}_{\boldsymbol{\eta}}\,\boldsymbol{\mathcal{M}}\boldsymbol{\eta}
- \boldsymbol{\mathcal{F}}_i'
-
\mathrm{ad}^{*}_{\boldsymbol{\xi}}\,\boldsymbol{\mathcal{F}}_i)
\,dX
= \\
&\int_{0}^{L}
\delta\boldsymbol{\zeta}^{\mathsf{T}}
(\bar{\boldsymbol{\mathcal{F}}}_a
+
\bar{\boldsymbol{\mathcal{F}}}_e)
\,dX .
\end{aligned}
\label{eq:weak_form}
\end{equation}
}

\noindent The weak statement in {\small \eqref{eq:weak_form}} is expressed in generalized coordinates using {\small $\delta\boldsymbol{\zeta}=\mathbf{J}\delta\mathbf{q}$} and {\small $\dot{\boldsymbol{\eta}}=\mathbf{J}\ddot{\mathbf{q}}+\dot{\mathbf{J}}\dot{\mathbf{q}}$}. Substitution into {\small \eqref{eq:weak_form}}, together with the load model, yields the generalized equation of motion:

{\small
\begin{equation}
\begin{aligned}
&\left[\int_{0}^{L}\mathbf{J}^{\mathsf{T}}\boldsymbol{\mathcal{M}}\mathbf{J}\,dX\right]\ddot{\mathbf{q}}
+
\left[\int_{0}^{L}\!\left(
\mathbf{J}^{\mathsf{T}}\boldsymbol{\mathcal{M}}\dot{\mathbf{J}}
+
\mathbf{J}^{\mathsf{T}}\mathrm{ad}^{*}_{\mathbf{J}\dot{\mathbf{q}}}
\boldsymbol{\mathcal{M}}\mathbf{J}
\right)\!dX\right]\dot{\mathbf{q}}-\\
&\int_{0}^{L}\mathbf{J}^{\mathsf{T}}
\left(
\boldsymbol{\mathcal{F}}_i'
+
\mathrm{ad}^{*}_{\boldsymbol{\xi}}\boldsymbol{\mathcal{F}}_i
\right)dX
=
\left[\int_{0}^{L}\mathbf{J}^{\mathsf{T}}
\bar{\boldsymbol{\mathcal{F}}}_a
dX\right]+
\mathbf{J}(\bar{X})^{\mathsf{T}}\boldsymbol{\mathcal{F}}_{p}\\
&+\left[\int_{0}^{L}\mathbf{J}^{\mathsf{T}}\boldsymbol{\mathcal{M}}
\mathrm{Ad}^{-1}_{\mathbf{g}}\,dX\right]
\mathrm{Ad}^{-1}_{\mathbf{g}_r}\boldsymbol{\mathcal{G}} .
\end{aligned}
\label{eq:generalized_dynamics}
\end{equation}
}

\noindent Grouping the integral terms in {\small (\ref{eq:generalized_dynamics})} gives the conventional form
{\small
\begin{equation}
\mathbf{M}\ddot{\mathbf{q}}
+\mathbf{C}\dot{\mathbf{q}}
+\mathbf{K}\mathbf{q}
+\mathbf{D}\dot{\mathbf{q}}
=\boldsymbol{\tau}
+\mathbf{F}_{\mathrm{ext}}
+\mathbf{F}_{g},
\label{eq:compact_dynamics}
\end{equation}
}

\noindent where {\small $\mathbf{M},\mathbf{C},\mathbf{K},\mathbf{D}\in\mathbb{R}^{n\times n}$} are the generalized inertia, Coriolis, stiffness, and damping matrices, respectively, and {\small $\boldsymbol{\tau},\mathbf{F}_{\mathrm{ext}},\mathbf{F}_{g}\in\mathbb{R}^{n}$} denote the generalized actuation input, external force, and gravity force, respectively. Further details on the derivation and construction of these generalized dynamic terms can be found in~\cite{renda2020geometric,hussain2021compliant,danesh2026real}.

\vspace{-1 em}

\subsection{Closed-Chain Constrained Dynamics}
The supportive architecture introduces loop-closure constraints because the coupled CRs are mechanically connected through the interaction structure. To derive the motion equations of this closed-chain system, the Lie-group formulation for CRs in \cite{armanini2021discrete} is adopted. The loop closure conditions are imposed at the velocity level in Pfaffian form as linear restrictions on the generalized velocity vector. Accordingly, the closed-chain kinematics satisfy {\small $\mathbf{A}(\mathbf{q})\,\dot{\mathbf{q}}=\mathbf{0},$} where {\small $\mathbf{A}(\mathbf{q})\in\mathbb{R}^{n_c\times n}$} is the constraint Jacobian and {\small $n_c$} is the number of constraints. The matrix {\small $\mathbf{A}$} is formed by enforcing kinematic compatibility at each closed-loop joint between frames attached to bodies {\small $A$} and {\small $B$}. For the {\small $i$}-th joint, the body Jacobians {\small $\mathbf{J}_{A_i}(\mathbf{q}),\mathbf{J}_{B_i}(\mathbf{q})\in\mathbb{R}^{6\times n}$} map the generalized velocity {\small $\dot{\mathbf{q}}$} to the corresponding body twists, expressed in the same joint coordinate frame. The relative twist across the joint is therefore
{\small
$\boldsymbol{\eta}_{i}
=
\big(\mathbf{J}_{A_i}(\mathbf{q})-\mathbf{J}_{B_i}(\mathbf{q})\big)\dot{\mathbf{q}}$}.

The joint constraint suppresses relative twist components that are incompatible with the joint type. These restrictions are represented by {\small$\mathbf{B}_{p,i}\in\mathbb{R}^{6\times n_{ci}}$}, whose columns span the joint constraint wrench subspace. By dual orthogonality, the admissible relative motion must be orthogonal to the constraint wrenches, leading to the Pfaffian constraint
{\small
$
\mathbf{B}_{p,i}^{\mathsf{T}}\boldsymbol{\eta}_{i}=0 $}. Replacing the relative twist with its Jacobian expression and assembling the Pfaffian constraints for all constrained joints yields the global constraint Jacobian
{\small
$
\mathbf{A}(\mathbf{q})
=
\begin{bmatrix}
\mathbf{B}_{p,i}^{\mathsf{T}}
\big(\mathbf{J}_{A_i}(\mathbf{q})-\mathbf{J}_{B_i}(\mathbf{q})\big)
\end{bmatrix}_{i=1}^{n_c}
\in \mathbb{R}^{n_c \times n}.$
}

The internal reaction wrenches required to maintain the closed-chain constraints are represented by the Lagrange multiplier vector {\small $\boldsymbol{\lambda}\in\mathbb{R}^{n_c}$}. The generalized dynamics are consequently augmented by the constraint force term as
{\small
\begin{equation}
\begin{aligned}
\mathbf{M}\ddot{\mathbf{q}}
+
\mathbf{C}\dot{\mathbf{q}}
+
\mathbf{K}\mathbf{q}
+
\mathbf{D}\dot{\mathbf{q}}
=
\boldsymbol{\tau}
+
\mathbf{F_{ext}}
+
\mathbf{F_{g}}
+
\mathbf{A}^{\mathsf{T}}\boldsymbol{\lambda}.
\label{eq:compact_dynamics_constrain}
\end{aligned}
\end{equation}
}

The coupled relation in {\small \eqref{eq:compact_dynamics_constrain}} can be solved explicitly for the Lagrange multipliers {\small $\boldsymbol{\lambda}$}, which can then be substituted into the dynamic equation to obtain a representation involving only the generalized coordinates. Here, an orthogonal projector onto the null space of the constraint Jacobian {\small $\mathbf{A}(\mathbf{q})$} is used to formulate the constrained dynamics, as stated below.

\begin{lemma}{\normalfont\cite{aghili2011projection}}
\label{lem:P_orth}

Let {\small $\mathbf{A}(\mathbf{q})\in\mathbb{R}^{n_c\times n}$} be the constraint Jacobian, and assume that it has full row rank. Define
{\small
\begin{align}
\mathbf{P}(\mathbf{q})
=
\mathbf{I}
-
\mathbf{A}^{+}(\mathbf{q})\,\mathbf{A}(\mathbf{q}),
\label{eq:P_def}
\end{align}
}

\noindent where {\small $\mathbf{A}^{+}$} denotes the Moore–Penrose inverse. Then {\small $\mathbf{P}$} is an orthogonal projector with the properties
{\small
\begin{align}
\mathbf{P}^{\mathsf{T}}=\mathbf{P},\qquad
\mathbf{P}^2=\mathbf{P},
\end{align}
}
and it annihilates the constrained directions {\small $ \mathbf{A}\mathbf{P}=
\mathbf{P}\mathbf{A}^{\mathsf{T}}=\mathbf{0}.
$
}
\end{lemma}

Premultiplication of the constrained dynamics {\small\eqref{eq:compact_dynamics_constrain}} by {\small $\mathbf{P}$} removes the constraint force contribution because {\small $\mathbf{P}\mathbf{A}^{\mathsf{T}}\boldsymbol{\lambda}=\mathbf{0}$}. The resulting projected dynamics are
{\small
\begin{align}
\mathbf{P}
\Big(
\mathbf{M}\ddot{\mathbf{q}}
+
\mathbf{C}\dot{\mathbf{q}}
+
\mathbf{K}\mathbf{q}
+
\mathbf{D}\dot{\mathbf{q}}
\Big)
=
\mathbf{P}
\Big(
\boldsymbol{\tau}
+
\mathbf{F_{ext}}
+
\mathbf{F}_{g}
\Big),
\label{eq:projected_dynamics}
\end{align}
}

\noindent which governs the system motion entirely within the tangent space of the constraint manifold.

\section{Constraint Consistent Control Design}

This section develops a constraint consistent control strategy for regulating
the Cartesian motion of the SCR while preserving the kinematic constraints of
the coupled system.

Let $\mathbf{x},\mathbf{x}_{d}\in\mathbb{R}^{3}$ denote the current and desired
Cartesian positions of the operative arm tip, respectively. The task-space
position and velocity errors are defined as
{\small
$
\mathbf{e}
=
\mathbf{x}_{d}-\mathbf{x},
\dot{\mathbf{e}}
=
\dot{\mathbf{x}}_{d}-\dot{\mathbf{x}},
$
} where $\dot{\mathbf{x}}$ is the operative-arm tip velocity obtained from the
GVS model.

\subsection{Projected Sliding Mode Control}

Let {\small$\mathbf{J}_{p}(\mathbf{q})\in\mathbb{R}^{3\times n}$} denote the translational Jacobian of the operative arm tip, such that
{\small
\begin{equation}
\dot{\mathbf{x}}=\mathbf{J}_{p}(\mathbf{q})\dot{\mathbf{q}} .
\label{eq:tip_velocity_jacobian}
\end{equation}
}

\noindent The controller is applied in a nonsingular operating region where {\small$\mathbf{J}_{p}$} remains full row rank. 
The Cartesian sliding variable is selected as
{\small
\begin{equation}
\mathbf{s}=\dot{\mathbf{e}}+\boldsymbol{\Gamma}\mathbf{e},
\label{eq:sliding_variable}
\end{equation}
}

\noindent where {\small$\boldsymbol{\Gamma}\in\mathbb{R}^{3\times3}$} is positive definite. If {\small$\mathbf{s}\rightarrow\mathbf{0}$}, then {\small$\dot{\mathbf{e}}+\boldsymbol{\Gamma}\mathbf{e}=\mathbf{0}$}, and the task-space error converges to zero.
The projected sliding mode control force, defined as
{\small$\boldsymbol{\tau}_{\mathrm{smc}}^{\mathrm{proj}}=\mathbf{P}\boldsymbol{\tau}_{\mathrm{smc}}$}, is given by

{\small
\begin{equation}
\begin{aligned}
\boldsymbol{\tau}_{\mathrm{smc}}^{\mathrm{proj}}
=
\mathbf{P}
\bigg[
&(\mathbf{C}+\mathbf{D})\dot{\mathbf{q}}
+\mathbf{K}\mathbf{q}
-\mathbf{F}_{\mathrm{ext}}
-\mathbf{F}_{g}
-\mathbf{M}\mathbf{J}_{p}^{+}\dot{\mathbf{J}}_{p}\dot{\mathbf{q}}  \\
&+
\mathbf{M}\mathbf{J}_{p}^{+}
\Big(
\ddot{\mathbf{x}}_{d}
+\boldsymbol{\Gamma}\dot{\mathbf{e}}
+\mathbf{K}_{s}\mathbf{s}
+\mathbf{R}_{H}\tanh(\boldsymbol{\Phi}^{-1}\mathbf{s})
\Big)
\bigg],
\end{aligned}
\label{eq:projected_smc_law}
\end{equation}
}

\noindent where {\small$\mathbf{K}_{s}\succ\mathbf{0}$} is the linear feedback gain, {\small$\mathbf{R}_{H}\succ\mathbf{0}$} is a diagonal switching gain matrix, and {\small$\boldsymbol{\Phi}\succ\mathbf{0}$} is a diagonal boundary layer matrix. The term {\small$\mathbf{R}_{H}\tanh(\boldsymbol{\Phi}^{-1}\mathbf{s})$} provides a bounded nonlinear reaching action, where the hyperbolic tangent function smooths the discontinuous switching action to reduce chattering near the sliding surface \cite{mancini2020sliding}. Moreover, {\small$\dot{\mathbf{J}}_{p}$} is the time derivative of the translational Jacobian, and {\small$\mathbf{J}_{p}^{+}$} is a right inverse of {\small$\mathbf{J}_{p}$}.

The following stability analysis considers the projected sliding mode control phase, {\small$t<t_{\mathrm{stiff,on}}$}, during which the operative arm tip is regulated to the desired position and the control input is given by {\small$\boldsymbol{\tau}^{\mathrm{proj}}=\boldsymbol{\tau}_{\mathrm{smc}}^{\mathrm{proj}}$}. The analysis establishes the asymptotic convergence of the task-space position error during this regulation phase. Once the desired position is reached, at {\small$t\geq t_{\mathrm{stiff,on}}$}, the active stiffness control term is introduced to modify the apparent force--displacement response of the SCR under external loading.

\begin{assumption}
\label{ass:task_rank}
The constrained task mapping remains nonsingular in the operating region, and {\small$\mathbf{J}_{p}^{+}$} is selected such that
{\small$\mathcal{R}(\mathbf{J}_{p}^{+})\subseteq\mathcal{R}(\mathbf{P})$}.
\end{assumption}

\begin{theorem}
\label{thm:projected_smc_stability}
Consider the projected closed-chain SCR dynamics under \eqref{eq:projected_smc_law}. If {\small$\boldsymbol{\Gamma}\succ\mathbf{0}$}, {\small$\mathbf{K}_{s}\succ\mathbf{0}$}, and {\small$\mathbf{R}_{H}$} is diagonal positive definite, then {\small$\mathbf{s}\rightarrow\mathbf{0}$} and {\small$\mathbf{e}\rightarrow\mathbf{0}$} asymptotically.
\end{theorem}

\begin{proof}
Substitution of \eqref{eq:projected_smc_law} into \eqref{eq:projected_dynamics} yields
{\small
\begin{equation}
\begin{aligned}
\mathbf{P}\mathbf{M}\ddot{\mathbf{q}}
=
\mathbf{P}\mathbf{M}\mathbf{J}_{p}^{+}
\Big(
&-\dot{\mathbf{J}}_{p}\dot{\mathbf{q}}
+\ddot{\mathbf{x}}_{d}
+\boldsymbol{\Gamma}\dot{\mathbf{e}}  \\
&+\mathbf{K}_{s}\mathbf{s}
+\mathbf{R}_{H}\tanh(\boldsymbol{\Phi}^{-1}\mathbf{s})
\Big).
\end{aligned}
\label{eq:proof_after_substitution}
\end{equation}
}
Since {\small$\ddot{\mathbf{x}}=\mathbf{J}_{p}\ddot{\mathbf{q}}+\dot{\mathbf{J}}_{p}\dot{\mathbf{q}}$} and
{\small$\dot{\mathbf{s}}=\ddot{\mathbf{x}}_{d}-\ddot{\mathbf{x}}+\boldsymbol{\Gamma}\dot{\mathbf{e}}$}, \eqref{eq:proof_after_substitution} gives
{\small
\begin{equation}
\mathbf{P}\mathbf{M}\mathbf{J}_{p}^{+}
\left(
\dot{\mathbf{s}}
+\mathbf{K}_{s}\mathbf{s}
+\mathbf{R}_{H}\tanh(\boldsymbol{\Phi}^{-1}\mathbf{s})
\right)
=
\mathbf{0}.
\label{eq:proof_projected_s_dynamics}
\end{equation}
}

The implication from \eqref{eq:proof_projected_s_dynamics} to the
sliding dynamics requires
{\small$\mathbf{P}\mathbf{M}\mathbf{J}_{p}^{+}$} to have full column rank. To
verify this property, assume by contradiction that
{\small$\mathbf{P}\mathbf{M}\mathbf{J}_{p}^{+}$} is rank deficient. Then, there
exists a nonzero vector $\boldsymbol{\alpha}\in\mathbb{R}^{3}$ such that
{\small
$
\mathbf{P}\mathbf{M}\mathbf{J}_{p}^{+}\boldsymbol{\alpha}
=\mathbf{0}.
$
}

\noindent Let {\small$\boldsymbol{\beta}=\mathbf{J}_{p}^{+}\boldsymbol{\alpha}$}; since {\small$\mathbf{J}_{p}^{+}$} is full column rank, {\small$\boldsymbol{\beta}\neq\mathbf{0}$}. Moreover, by Assumption~\ref{ass:task_rank}, {\small$\mathbf{P}\boldsymbol{\beta}=\boldsymbol{\beta}$}. Therefore,
{\small
$
0
=
\boldsymbol{\beta}^{\mathsf{T}}\mathbf{P}\mathbf{M}\boldsymbol{\beta}
=
\boldsymbol{\beta}^{\mathsf{T}}\mathbf{M}\boldsymbol{\beta},
$
} which contradicts $\mathbf{M}\succ\mathbf{0}$ and $\boldsymbol{\beta}\neq\mathbf{0}$. Hence, $\mathbf{P}\mathbf{M}\mathbf{J}_{p}^{+}$ is full column rank, and \eqref{eq:proof_projected_s_dynamics} implies
{\small
\begin{equation}
\dot{\mathbf{s}}
+\mathbf{K}_{s}\mathbf{s}
+\mathbf{R}_{H}\tanh(\boldsymbol{\Phi}^{-1}\mathbf{s})
=
\mathbf{0}.
\label{eq:proof_s_dynamics}
\end{equation}
}

\noindent Consider $V=\frac{1}{2}\mathbf{s}^{\mathsf{T}}\mathbf{s}$. Along \eqref{eq:proof_s_dynamics},
{\small
\begin{equation}
\dot{V}
=
-\mathbf{s}^{\mathsf{T}}\mathbf{K}_{s}\mathbf{s}
-\mathbf{s}^{\mathsf{T}}\mathbf{R}_{H}\tanh(\boldsymbol{\Phi}^{-1}\mathbf{s})
\leq
-\lambda_{\min}(\mathbf{K}_{s})\|\mathbf{s}\|^{2}.
\label{eq:proof_Vdot}
\end{equation}
}

\noindent The second term is nonnegative because $\mathbf{R}_{H}$ and
$\boldsymbol{\Phi}$ are positive diagonal matrices and
$s_i\tanh(s_i/\phi_i)\geq0$. Thus, $\dot{V}<0$ for
$\mathbf{s}\neq\mathbf{0}$, which implies that $\mathbf{s}$ converges
asymptotically to zero. From \eqref{eq:sliding_variable},
$\dot{\mathbf{e}}+\boldsymbol{\Gamma}\mathbf{e}=\mathbf{s}$; since
$\boldsymbol{\Gamma}\succ\mathbf{0}$ and
$\mathbf{s}\rightarrow\mathbf{0}$ asymptotically, the tracking error
$\mathbf{e}$ also converges asymptotically to zero.
\end{proof}

\subsection{Projected Active Apparent Stiffness Control}

After the operative arm tip reaches the desired operating configuration, an active Cartesian force is introduced to modify the apparent stiffness of the system. This force is generated from tip-position feedback and does not require direct force measurements. The external loads considered in this study are known in advance. The force is computed from the position error as

{\small
\begin{equation}
\mathbf{F}_{\mathrm{app}}=
\mathbf{K}_{\mathrm{app}}\mathbf{e}
,
\qquad
t\geq t_{\mathrm{stiff,on}},
\label{eq:apparent_stiffness_force}
\end{equation}
}

\noindent where {\small$\mathbf{K}_{\mathrm{app}}\in\mathbb{R}^{3\times3}$} is the apparent stiffness gain matrix. Prior to activation of the stiffness control term,

{\small
\begin{equation}
\mathbf{F}_{\mathrm{app}}=
\mathbf{0},
\qquad
t<t_{\mathrm{stiff,on}}.
\label{eq:inactive_apparent_stiffness}
\end{equation}
}

The Cartesian apparent stiffness force is mapped to the generalized coordinate space and projected to preserve constraint consistency:

{\small
\begin{equation}
\boldsymbol{\tau}_{\mathrm{app}}^{\mathrm{proj}}=
\mathbf{P}
\mathbf{J}_{p}^{\mathsf{T}}
\mathbf{F}_{\mathrm{app}},
\qquad
t\geq t_{\mathrm{stiff,on}}.
\label{eq:apparent_stiffness_generalized_force}
\end{equation}
}

\noindent Therefore, the total generalized control input is given by
{\small
$
\boldsymbol{\tau}^{\mathrm{proj}}=
\boldsymbol{\tau}_{\mathrm{smc}}^{\mathrm{proj}}
+
\boldsymbol{\tau}_{\mathrm{app}}^{\mathrm{proj}}.
\label{eq:total_control_force}
$
}

Assuming that the generalized control input is realizable by the tendon actuation system, i.e.,
{\small$\boldsymbol{\tau}^{\mathrm{proj}}\in\mathcal{R}(\mathbf{B}_{q})$}, it is converted into tendon tensions through the generalized actuation matrix $\mathbf{B}_{q}$ as
{\small
$
\mathbf{u}
=
\mathbf{B}_{q}^{+}
\boldsymbol{\tau}^{\mathrm{proj}},
$
} where {\small$\mathbf{u}\in\mathbb{R}^{n_{\mathrm{act}}}$} is the tendon tension vector, {\small$n_{\mathrm{act}}$} is the number of actuators, and {\small$\mathbf{B}_{q}^{+}$} denotes the Moore--Penrose pseudoinverse of the generalized actuation matrix.

It should be noted that {\small$\mathbf{K}_{\mathrm{app}}$} represents the commanded stiffness gain rather than the achieved stiffness of the SCR. For a fixed operating configuration, the achieved directional stiffness increases approximately linearly with this gain. Therefore, a desired stiffness can be obtained by identifying the slope and intercept of this relationship through calibration and selecting the corresponding gain.

\begin{remark}
\label{rem:bounded_stiffness_input}
The boundedness of the apparent stiffness input is considered within the
prescribed operating region. Assume that the SCR remains in a compact
operating set in which positive constants {\small$\bar{e}$} and {\small$\bar{J}_{p}$}
exist such that
{\small
$
\|\mathbf{e}(t)\|_{2}\leq\bar{e},
\|\mathbf{J}_{p}(\mathbf{q}(t))\|_{2}\leq\bar{J}_{p}$}, for
{\small
$t\geq t_{\mathrm{stiff,on}}.$
}
Since {\small$\mathbf{P}$} is an orthogonal projector,
{\small$\|\mathbf{P}\|_{2}\leq1$}, and the projected apparent stiffness input
satisfies
{\small
$
\left\|
\boldsymbol{\tau}_{\mathrm{app}}^{\mathrm{proj}}
\right\|_{2}
=
\left\|
\mathbf{P}\mathbf{J}_{p}^{\mathsf{T}}
\mathbf{K}_{\mathrm{app}}\mathbf{e}
\right\|_{2} 
\leq
\bar{J}_{p}
\left\|\mathbf{K}_{\mathrm{app}}\right\|_{2}
\bar{e}.
$
}
Therefore, the apparent stiffness input remains bounded and
constraint consistent within the prescribed operating region. This
conditional boundedness result does not establish asymptotic stability
of the complete closed-loop system following activation of the
apparent stiffness controller.

\end{remark}

\vspace{-1 em}
\section{Simulation Results}
\label{sec:sim_results}

The proposed control framework was initially evaluated through numerical simulations of the SCR system. The simulated platform consists of two Nitinol CRs: an operative arm and a supportive arm. The operative arm serves as the primary task arm, and its tip position is selected as the controlled Cartesian output. The supportive arm is mechanically coupled to the operative arm through a rigid closed-loop joint located \(0.42~\mathrm{m}\) from the operative arm base. The connection is arranged with a relative attachment angle of \(90^\circ\), forming a closed-chain supportive structure. The structural and material parameters used in the simulations are summarized in Table~\ref{tab:scr_structural_parameters}.

The densities reported in Table~\ref{tab:scr_structural_parameters} correspond to the effective distributed densities used in the numerical model. These values account for the combined contribution of the Nitinol backbone, tendons, disks, and additional attached components, rather than the backbone material alone. Therefore, the operative and supportive arms are assigned different effective densities according to their physical mass distributions.

\begin{table}[t]
\centering
\caption{Physical parameters of the simulated SCR system.}
\label{tab:scr_structural_parameters}
\setlength{\tabcolsep}{4pt}
\renewcommand{\arraystretch}{1.05}
\begin{tabular}{lll}
\toprule
\textbf{Component} & \textbf{Parameter} & \textbf{Value} \\
\midrule
\multirow{5}{*}{Both CRs}
& Material & Nitinol \\
& Length & $L_o=L_s=0.6~\mathrm{m}$ \\
& Young's modulus & $E_o=E_s=50~\mathrm{GPa}$ \\
& Poisson's ratio & $\nu_o=\nu_s=0.3$ \\
& Radius & $r_o=r_s=0.75565~\mathrm{mm}$ \\
\midrule
Operative CR 
& Effective density & $\rho_o=56,211~\mathrm{kg/m^3}$ \\
Supportive CR 
& Effective density & $\rho_s=50,636~\mathrm{kg/m^3}$ \\
\bottomrule
\end{tabular}
\end{table}

For the numerical implementation, each CR is modeled as a single soft link with one division. The strain field of each arm is approximated using Legendre polynomial basis functions. Linear bending is retained in two orthogonal bending directions for each arm, resulting in four generalized coordinates per CR. Since the SCR system consists of one operative arm and one supportive arm, the complete closed-chain model has a total of \(n=8\) generalized coordinates.

The distributed terms of the arms are evaluated using
\(11\) Gaussian quadrature points along each backbone. The closed-chain
connection is imposed through the loop-closure constraint Jacobian, and the
projection matrix ensures that the numerical motion remains compatible with
the rigid coupling between the two arms.

The task-space output is chosen as the Cartesian position of the operative arm
tip, $    \mathbf{x}
    =
    \begin{bmatrix}
    x & y & z
    \end{bmatrix}^{\mathsf{T}}
    \in \mathbb{R}^{3}.
$
For each simulation case, the SCR is first driven to the prescribed desired
position using the projected sliding mode controller. After the regulation
phase is completed, the active apparent stiffness term is activated at
\(t=t_{\mathrm{stiff,on}}\). A known external point load is then applied at the operative arm tip at
\(t=t_{\mathrm{load}}\). The load acts in the gravity direction, which
corresponds to the \(x\)-direction in the simulation coordinate frame.

The steady-state tip positions before and after load application are denoted by
{\small $\mathbf{x}_{\mathrm{unloaded}}$} and
{\small $\mathbf{x}_{\mathrm{loaded}}$}, respectively. The load-induced
displacement vector is defined as
{\small $\Delta\mathbf{x}
=\mathbf{x}_{\mathrm{loaded}}-\mathbf{x}_{\mathrm{unloaded}}$}.
Since the stiffness response is evaluated in the loading direction, the
corresponding {\small $x$}-direction deflection is obtained as
{\small $\Delta x
=\left|x_{\mathrm{loaded}}-x_{\mathrm{unloaded}}\right|$}.
The directional stiffness is then computed using
{\small $K_x=F_{\mathrm{load}}/\Delta x$}, where the applied force is given by
{\small $F_{\mathrm{load}}=mg$}. Therefore, {\small $\Delta x$} quantifies the
load-induced displacement, while {\small $K_x$} provides the corresponding
apparent directional stiffness.

For this representative case, the initial operative arm tip position is
{\small
\(\mathbf{x}_{0}=[0.57736,\;0.15336,\;0]^{\mathsf{T}}~\mathrm{m}\)
},
and Desired Point~1 is defined as
{\small
\(\mathbf{x}_{d,1}=[0.33262,\;0.20213,\;-0.01305]^{\mathsf{T}}~\mathrm{m}\)
}.
Fig.~\ref{fig:sim_time_response_d1_100g_kapp5000} shows the corresponding
simulation response under a \(100\)~g load with
\(K_{\mathrm{app},x}=5~\mathrm{N/mm}\). The total simulation duration is
\(6\)~s. The active stiffness term is turned on at \(t=3\)~s, and the external
load is applied at \(t=3.5\)~s.

As shown in Fig.~\ref{fig:sim_time_response_d1_100g_kapp5000}(a), the
Cartesian errors converge rapidly during the initial regulation phase. All error components
approach zero within a short time interval, confirming that the projected
sliding mode controller drives the operative arm tip to the desired
configuration while preserving the closed-chain constraints.

Fig.~\ref{fig:sim_time_response_d1_100g_kapp5000}(b) shows the active apparent
stiffness force. When the stiffness controller is activated at \(t=3\)~s, the
controller generates a task-space force according to the prescribed stiffness
gain. After the external load is applied at \(t=3.5\)~s, the apparent stiffness
force increases mainly in the \(x\)-direction to oppose the load-induced
displacement, while the \(y\)- and \(z\)-components remain close to zero. The
force response exhibits a short transient and then settles to a steady value.

The corresponding tendon inputs are shown in
Fig.~\ref{fig:sim_time_response_d1_100g_kapp5000}(c). The tendon inputs are constrained within the saturation bounds of
\([-50,50]\)~N. During the initial regulation stage, some inputs reach these
bounds due to the large initial position error; however, after the transient
response, the inputs decrease and remain bounded within the prescribed limits.
Only small changes occur when the stiffness term is activated, whereas the load
application produces a clear adjustment in the tendon inputs to compensate for
the external force. These results indicate that the proposed stiffness
modulation modifies the load response of the SCR without destabilizing the
closed-chain system.

The effect of the commanded apparent stiffness was then studied by varying
\(K_{\mathrm{app},x}\) under different known external loads.
Fig.~\ref{fig:sim_stiffness_load_comparison} shows the resulting
\(x\)-direction deflection and the corresponding directional stiffness for
\(20\), \(50\), and \(100\)~g loads. As \(K_{\mathrm{app},x}\) increases, the
load-induced deflection decreases for all load cases. At the same time, the
computed directional stiffness increases, indicating that the controller makes
the SCR respond more stiffly in the loading direction. This trend confirms
that the active stiffness term can tune the apparent force--displacement
behavior of the closed-chain SCR. 
\noindent It should be noted that $K_x$ represents the effective directional stiffness computed from the applied load and the corresponding displacement between the unloaded and loaded equilibrium configurations. Although $K_{\mathrm{app},x}$ is fixed, the SCR exhibits geometrically nonlinear and configuration-dependent behavior. Different loads therefore produce different equilibrium configurations and modify the local closed-chain kinematics and force--displacement response. Consequently, the achieved directional stiffness can vary with the applied load.

To examine whether the proposed stiffness control is maintained across
different robot configurations, a \(100\)~g load was applied after regulating
the SCR to three different desired positions. In all cases, the operative arm
tip started from the same initial position.
The desired positions were selected as
{\small
\(\mathbf{x}_{d,1}=[0.33262,\;0.20213,\;-0.01305]^{\mathsf{T}}~\mathrm{m}\)
},
{\small
\(\mathbf{x}_{d,2}=[0.51060,\;-0.00611,\;-0.00253]^{\mathsf{T}}~\mathrm{m}\)
},
and
{\small
\(\mathbf{x}_{d,3}=[0.45868,\;0.24256,\;-0.02470]^{\mathsf{T}}~\mathrm{m}\)
}.

Fig.~\ref{fig:stiffness_control} compares the \(x\)-direction deflection and
the corresponding directional stiffness for these three operating
configurations as \(K_{\mathrm{app},x}\) is varied. The results show that the
SCR exhibits different inherent stiffness values at different desired
positions, since the compliance of the closed-chain structure depends on its
configuration. Therefore, the load-induced deflection is not identical for the
three desired points. However, despite these configuration-dependent stiffness
differences, the same trend is observed in all cases: increasing
\(K_{\mathrm{app},x}\) reduces the load-induced displacement and increases the
measured directional stiffness. This demonstrates that the proposed active stiffness controller can modify the apparent stiffness of the SCR across different operating configurations. Moreover, the simulation results show an approximately linear relationship between the commanded gain $K_{\mathrm{app},x}$ and the achieved directional stiffness $K_x$ for each configuration. Thus, a configuration specific linear model can be fitted as
$
K_x=aK_{\mathrm{app},x}+b,
$
 where $a$ and $b$ are identified from the simulation data. The gain required to achieve a desired stiffness $K_{x,d}$ is then $K_{\mathrm{app},x}
=
\frac{K_{x,d}-b}{a}$.

\begin{figure}
    \centering
    \includegraphics[width=0.4\textwidth]{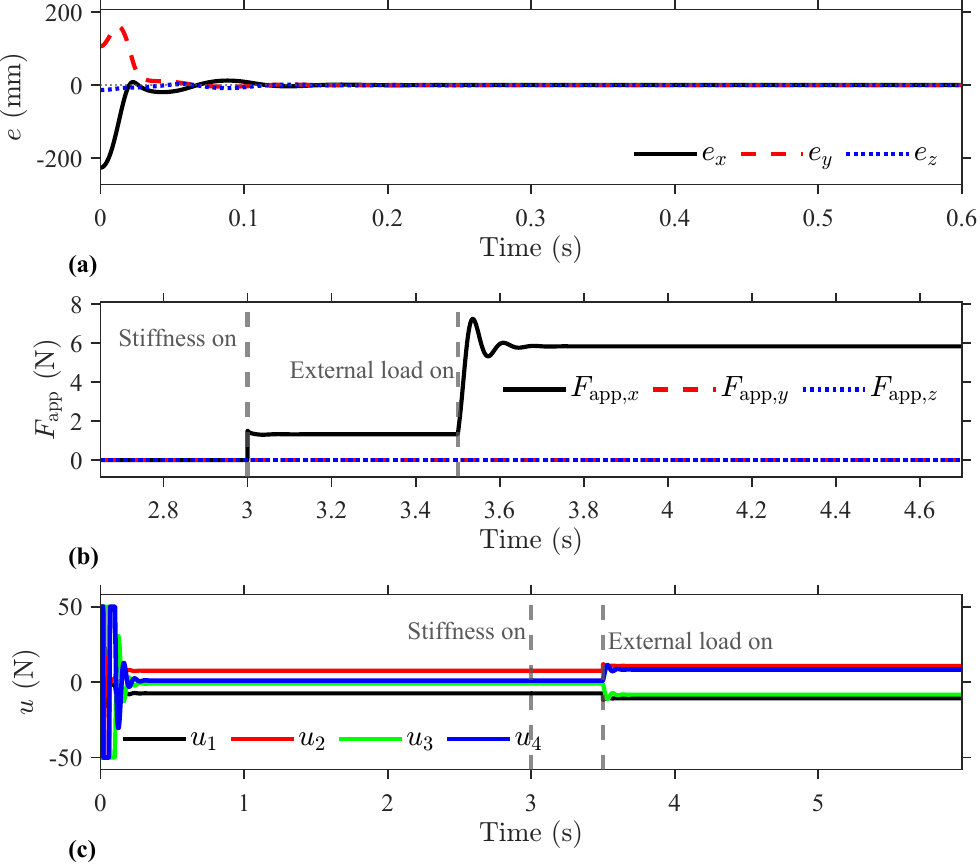}
\caption{Simulation response for Desired Point~1 under a $100$~g load and $K_{\mathrm{app},x}=5$~N/mm: (a) Cartesian errors, (b) active apparent stiffness force, and (c) tendon inputs. Gray lines indicate stiffness activation and load application.}
\label{fig:sim_time_response_d1_100g_kapp5000}
\end{figure}

\begin{figure}
    \centering
    \includegraphics[width=0.48\textwidth]{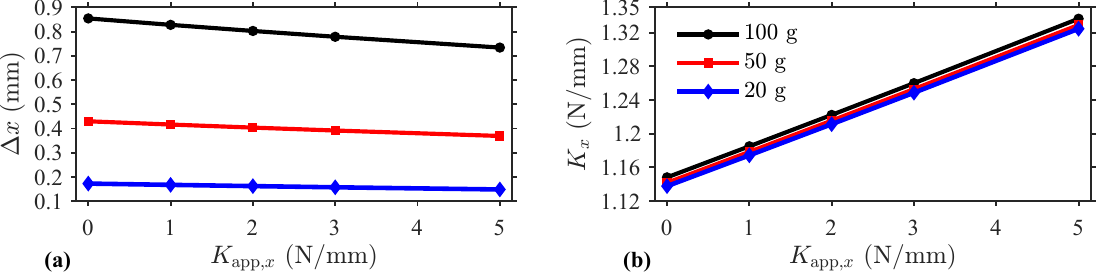}
\caption{Effect of the apparent stiffness $K_{\mathrm{app},x}$ on (a) the load-induced tip deflection in the $x$ direction and (b) the computed directional stiffness under external loads of $20$, $50$, and $100$~g.}
\label{fig:sim_stiffness_load_comparison}
\end{figure}

\begin{figure}
    \centering
    \includegraphics[width=0.43\textwidth]{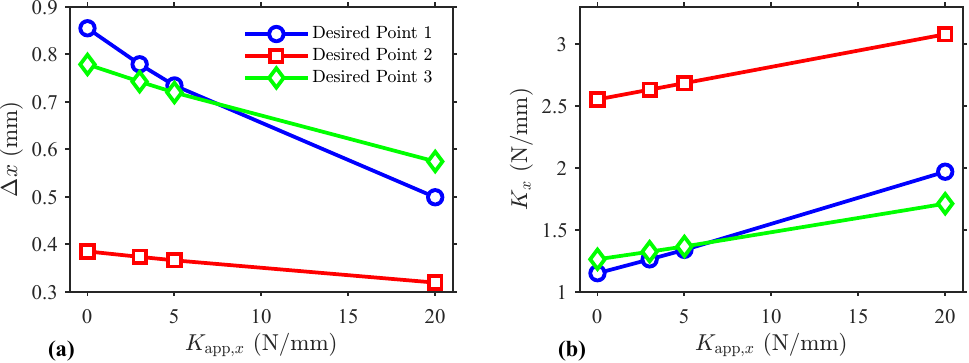}
\caption{Effect of $K_{\mathrm{app},x}$ on (a) the $x$-direction tip deflection and (b) the measured directional stiffness under a $100$~g load at three desired positions.}
    \label{fig:stiffness_control}
\end{figure}

\section{Experimental Results}
\label{sec:experimental_results}

This section presents the experimental validation of the proposed active stiffness control under known external loads and different desired configurations.

\subsection{Experimental Setup}

The proposed stiffness control framework was experimentally validated on the SCR platform shown in Fig.~\ref{fig:scr_system}(b). The platform consists of two tendon-driven Nitinol CRs with bases separated by \(0.5~\mathrm{m}\). Kevlar tendons are routed parallel to the backbone through spacer disks and are distributed at \(45^\circ\) intervals around the backbone to generate bending through motor actuation. The supportive arm is rigidly connected to the operative arm at \(0.42~\mathrm{m}\) from the operative arm base, with a \(90^\circ\) connection angle. The same geometry was used in simulation and experiment.

The operative arm tip position was measured in real time using a Tracker 3.0 Vicon motion capture system (Bilston, UK). Reflective markers were attached to the robot, and the measured marker positions were used to reconstruct the Cartesian tip position. In each experiment, the projected sliding mode controller first regulated the tip to the desired position. After regulation, the active stiffness term was activated, and a known weight was hung at the operative arm tip. The resulting steady-state displacement was then measured and used to compute the directional stiffness.

The control input was first obtained as tendon tension from the GVS dynamic model. A tendon actuation function governed by the model dynamics was then used to compute the corresponding tendon displacements \cite{danesh2026real}. These displacements were converted into motor position commands using the known motor shaft diameter and sent to the Dynamixel AX-12A servo motors (Robotis, Seoul, Korea).

\vspace{-0.7 em}
\subsection{Stiffness Control Under Different Loads}

Fig.~\ref{fig:experimental_configurations} shows the experimental configurations used for stiffness validation. The SCR is first driven from the initial configuration in Fig.~\ref{fig:experimental_configurations}(a) to the desired operating point. After the operative arm tip reaches the desired point, the active stiffness controller is activated, and an external load is applied at the tip.

The first experimental study evaluates the effect of the commanded apparent stiffness at Desired Point~A, shown in Fig.~\ref{fig:experimental_configurations}(b) and Fig.~\ref{fig:experimental_configurations}(c), under two external loads. Fig.~\ref{fig:exp_pointA_50g_100g_stiffness} reports the resulting \(x\)-direction deflection and the corresponding directional stiffness for \(50\)~g and \(100\)~g loads as \(K_{\mathrm{app},x}\) is varied.

As expected, the larger load produces a larger tip deflection. However, for both load cases, increasing \(K_{\mathrm{app},x}\) reduces the measured load-induced displacement. This confirms that the active stiffness term increases the resistance of the operative arm tip against the applied load. The corresponding directional stiffness increases with \(K_{\mathrm{app},x}\), showing that the proposed controller changes the apparent force--displacement behavior of the SCR. The error bars show the standard deviation across three repeated trials, indicating the repeatability of the experimental measurements.

These experimental results are consistent with the simulation findings, showing that increasing the active stiffness gain reduces the load-induced tip deflection and increases the apparent directional stiffness of the SCR.

\subsection{Configuration-Dependent Stiffness Modulation}

The second experimental study evaluates the stiffness response at the two configurations shown in Fig.~\ref{fig:experimental_configurations}. Desired Point~A is shown before and after loading in Fig.~\ref{fig:experimental_configurations}(b) and (c), while Desired Point~B is shown in Fig.~\ref{fig:experimental_configurations}(d) and (e). A \(100\)~g load was applied at the operative arm tip in both cases, and Fig.~\ref{fig:exp_stiffness_100g} compares the resulting \(x\)-direction deflection and directional stiffness.

The results show that the robot has different inherent stiffness at the two desired configurations. This is expected because the compliance of a closed-chain SCR depends on its configuration. Therefore, the same external load does not produce the same deflection at Desired Point~A and Desired Point~B. However, despite this configuration-dependent behavior, increasing \(K_{\mathrm{app},x}\) consistently reduces the load-induced displacement at both desired points. The measured directional stiffness also increases with the commanded apparent-stiffness gain.

These results demonstrate that the proposed controller can actively modify the apparent stiffness of the SCR from different operating configurations. In other words, even when the robot starts from configurations with different passive stiffness, the active stiffness term can increase the measured directional stiffness and reduce the load-induced tip displacement.

The experimental relationship between $K_{\mathrm{app},x}$ and the achieved directional stiffness exhibits greater variability than the nearly linear trend observed in simulation. This difference is attributed to unmodeled physical effects, including tendon friction and hysteresis, actuator resolution, measurement uncertainty, and small deviations in the applied loading direction. These factors influence the transmitted tendon forces and measured tip displacement. Nevertheless, the experimental results consistently demonstrate an increase in the achieved directional stiffness with increasing $K_{\mathrm{app},x}$.

\begin{figure}
    \centering
    \includegraphics[width=0.4\textwidth,trim=3.5cm 0.4cm 3.5cm 0.4cm, clip]{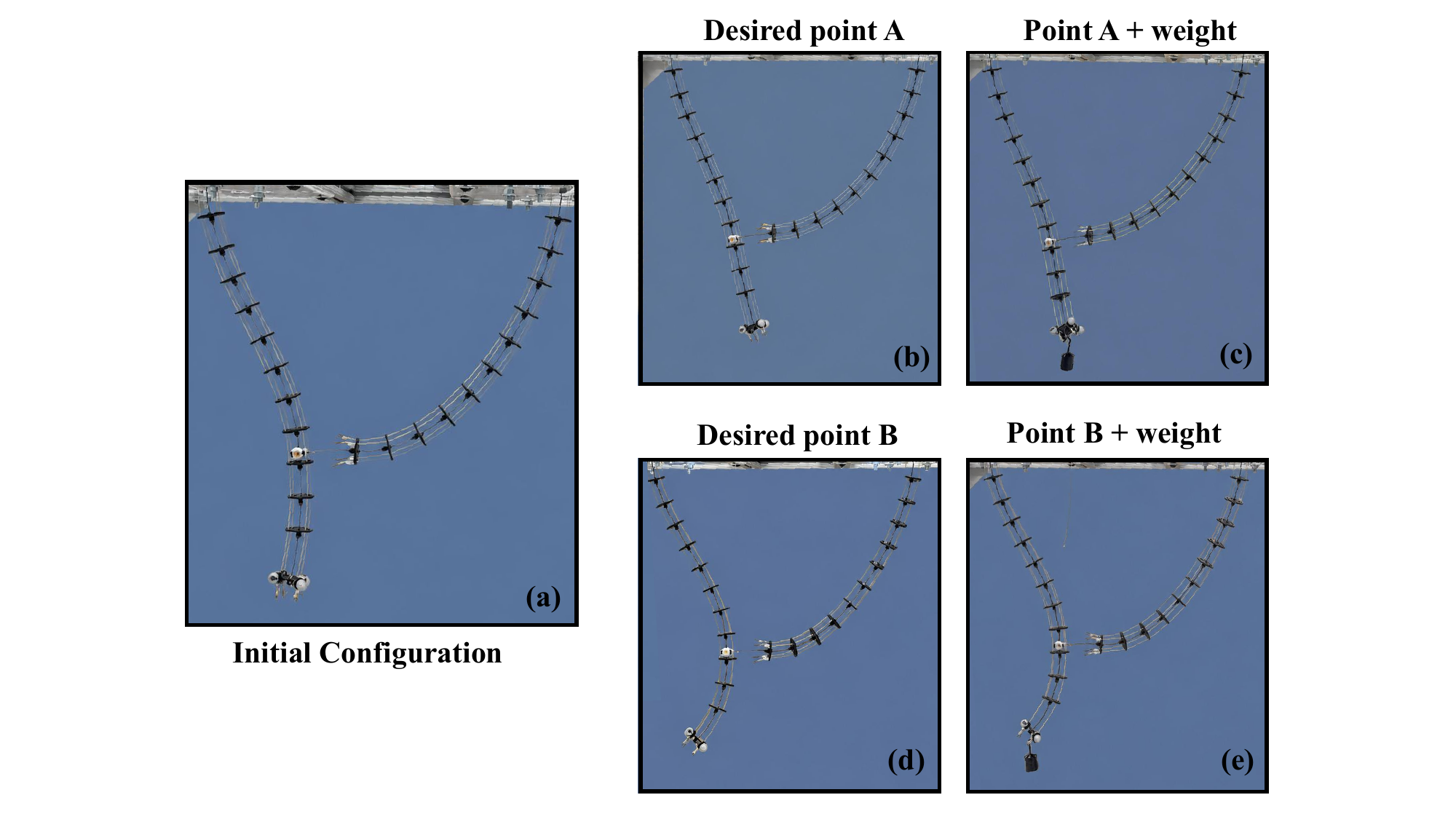}
\caption{Experimental configurations of the SCR: 
(a) initial configuration, (b) desired point A, (c) desired point A under external load, 
(d) desired point B, and (e) desired point B under external load.}
\label{fig:experimental_configurations}
    \label{fig:setup_conf}
\end{figure}

\begin{figure}
    \centering
    \includegraphics[width=0.48\textwidth]{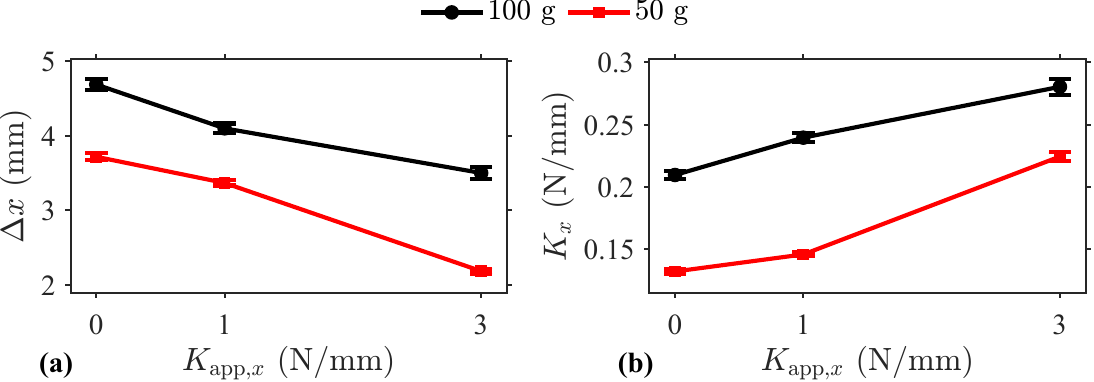}
\caption{Experimental effect of the commanded apparent stiffness $K_{\mathrm{app},x}$ on (a) the $x$-direction tip deflection and (b) the directional stiffness at Desired Point A under $50$~g and $100$~g loads. Error bars denote the standard deviation over three trials.}
\label{fig:exp_pointA_50g_100g_stiffness}
\end{figure}

\begin{figure}
    \centering
    \includegraphics[width=0.48\textwidth]{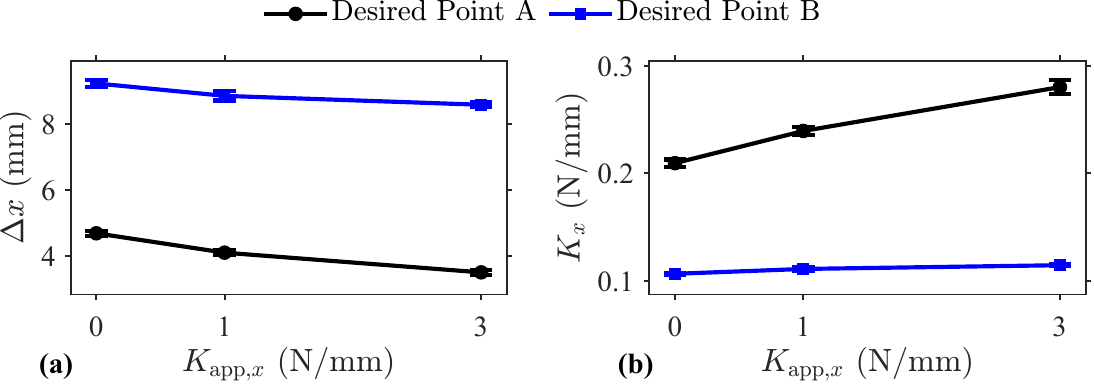}
\caption{Experimental deflection and directional stiffness versus $K_{\mathrm{app},x}$ under a $100$~g load. Error bars show standard deviation over $n=3$ trials.}
\label{fig:exp_stiffness_100g}
\end{figure}

Table~\ref{tab:stiffness_improvement} compares representative simulation and
experimental results with the corresponding SMC-only baselines, where
{\small$K_{\mathrm{app},x}=0$}. In simulation, {\small $K_{\mathrm{app},x}=20~\mathrm{N/mm}$}
reduced the deflection by {\small$41.64\%$} and increased the directional stiffness by
{\small$71.35\%$}. Experimentally, {\small$K_{\mathrm{app},x}=3~\mathrm{N/mm}$} reduced the
deflection by {\small$25.27\%\pm0.56\%$} and increased the stiffness by
{\small$33.82\%\pm0.99\%$}. Since the response is configuration dependent, other
workspace locations may exhibit different or larger improvements.

\begin{table}
\centering
\caption{Performance improvement relative to the SMC-only case.}
\label{tab:stiffness_improvement}
\scriptsize
\setlength{\tabcolsep}{12pt}
\renewcommand{\arraystretch}{0.82}
\begin{tabular}{cccc}
\toprule
\textbf{Case} &
\textbf{\shortstack{$K_{\mathrm{app},x}$\\(N/mm)}} &
\textbf{\shortstack{$\Delta x$ reduction\\(\%)}} &
\textbf{\shortstack{$K_x$ increase\\(\%)}} \\
\midrule
\multirow{3}{*}{Simulation\ Point 1}
& 3  & 8.87  & 9.74 \\
& 5  & 14.09 & 16.40 \\
& 20 & 41.64 & 71.35 \\
\midrule
\multirow{2}{*}{Experiment \ Point A}
& 1 & $12.50\pm2.81$ & $14.36\pm3.65$ \\
& 3 & $25.27\pm0.56$ & $33.82\pm0.99$ \\
\bottomrule
\end{tabular}
\end{table}

\section{Conclusion}
\label{sec:conclusion}

This paper presented an active task-space stiffness control framework for a closed-chain SCR. The SCR was modeled using the GVS formulation, and the closed-chain constraints were handled through projection. A projected sliding mode controller was used to regulate the operative arm tip to the desired position, after which an active apparent stiffness term was activated to modify the load response of the system.

Simulation and experimental results showed that increasing the stiffness control gain reduces the load-induced tip deflection and increases the apparent directional stiffness. The results also confirmed that the SCR has configuration-dependent stiffness, but the proposed controller consistently improved the stiffness response at different desired positions. These findings validate the proposed framework for active stiffness control of supportive closed-chain CRs without changing the physical structure of the robot.

\bibliographystyle{IEEEtran}
\bibliography{references}

\end{document}